\documentclass[sigconf]{acmart}

\usepackage{xcolor}
\usepackage{soul}
\usepackage{multirow,tabularx}
\usepackage{caption}
\usepackage{subcaption}
  \usepackage{placeins}
\usepackage{algorithmic}
\usepackage{tabularx}
\usepackage{graphicx}
\usepackage[ruled,vlined,linesnumbered]{algorithm2e}

\AtBeginDocument{%
  \providecommand\BibTeX{{%
    \normalfont B\kern-0.5em{\scshape i\kern-0.25em b}\kern-0.8em\TeX}}}

\copyrightyear{2024}
\acmYear{2024}

\acmConference[]{ACM Conference}{2024}{University Park, PA, USA}

\begin{document}

\title{WildGraph: Realistic Long-Horizon Trajectory Generation with Limited Sample Size}

\author{Ali Al-Lawati}
\email{aha112@psu.edu}
\orcid{1234-5678-9012}
\affiliation{%
  \institution{The Pennsylvania State University}
  \city{University Park}
  \state{PA}
  \country{USA}
  \postcode{16802}
}

\author{Elsayed Eshra}
\email{eme5375@psu.edu}

\affiliation{%
  \institution{The Pennsylvania State University}
  \city{University Park}
  \state{PA}
  \country{USA}
  \postcode{16802}
}

\author{Prasenjit Mitra}
\email{pmitra@psu.edu}
\affiliation{%
  \institution{The Pennsylvania State University}
  \city{University Park}
  \state{PA}
  \country{USA}
  \postcode{16802}
}

\renewcommand{\shortauthors}{}

\begin{abstract}

Trajectory generation is an important task in movement studies. Generated trajectories augment the training corpus of deep learning applications, facilitate experimental and theoretical research, and mitigate the privacy concerns associated with real trajectories. This is especially significant in the wildlife domain, where trajectories are scarce due to the ethical and technical constraints of the collection process. In this paper, we consider the problem of generating long-horizon trajectories, akin to wildlife migration, based on a small set of real samples. We propose a hierarchical approach to learn the global movement characteristics of the real dataset, and recursively refine localized regions. Our solution, WildGraph discretizes the geographic path into a prototype network of H3\footnote{https://www.uber.com/blog/h3/} regions and leverages a novel recurrent VAE to probabilistically generate paths over the regions, based on occupancy. Experiments performed on two wildlife migration datasets demonstrate the remarkable capability of WildGraph to generate realistic months-long trajectories using a sample size as small as 60 while improving generalization compared to existing work. Moreover, WildGraph achieves superior or comparable performance on performance measures, including geographic imagery similarity. Our code is published on the following repository: \url{https://github.com/aliwister/wildgraph}.

\end{abstract}




\maketitle

\section{Introduction}

Recent advances in wildlife tracking have played a key role in contributing spatial trajectory datasets to wildlife researchers. These datasets provide invaluable insights into the patterns and behaviors of movement, enabling a deeper understanding of the underlying mechanisms driving wildlife movement. In particular, wildlife trajectories are utilized by scientists across different fields to promote wildlife conservation~\cite{wall_novel_2014}, anti-poaching efforts~\cite{park_ape_2015}, mitigating human-wildlife conflict~\cite{buchholtz_using_2020}, and to help guide policy towards wildlife preservation~\cite{tomkiewicz_global_2010}.

However, most published datasets are scarce, i.e. they include a small sample set of animals. Wildlife researchers have cited ethical and technical constraints involved in wildlife trajectory data collection~\cite{dore_review_2020}. In particular, data is traditionally collected using devices attached directly to the subject, which requires capturing and possibly sedating it. The devices are further constrained by the animal's size, and may be prone to intermittent or permanent failures, thus limiting the availability of trajectory data. 

To overcome the challenges associated with collecting real movement data, trajectory generation has been considered in various domains, including human mobility~\cite{long_practical_2023, yuan_activity_2022, rossi_human_2021}, vehicle mobility~\cite{cao_generating_2021, rossi_vehicle_2021}, as well as wildlife~\cite{al-lawati_wildgen_2023, technitis_b_2015} movement studies.  Generated trajectories help enrich the training corpus to advance deep learning (DL) models~\cite{castelli2023enhancing}, and identify trends and patterns. For example, with trajectory augmentation~\cite{zhou2021improving, 10.1145/3615885.3628008}, known past poaching incidents mapped against a large trajectory set ($real$ $ \cup$ $generated$) help enforcement efforts to target high-risk regions~\cite{9101618}. Furthermore, trajectory generation plays a key role in addressing privacy concerns of releasing real data, which may put wildlife at the risk of poaching~\cite{tig}. Generated trajectories also provide valuable insight into experimental studies, scenario planning, and risk assessment.  Researchers can simulate the effects of policies and environmental decisions, thereby improving the robustness and generalization of their models and analyses~\cite{beach_scenario_2015, catano_using_2015}.

Generative models powered by deep learning algorithms have demonstrated impressive capabilities in learning the knowledge semantics of a dataset in various domains and capturing it to produce original data based on the learned patterns. Simple methods based on variational autoencoders (VAEs) have shown a tremendous capability to generate meaningful samples based on a small data sample, i.e. the {\bf small data problem}. Ideally, we want the generated set of trajectories to have the same characteristics (measured by, e.g. Hausdorff Distance or correlation coefficient) of a set of real trajectories, and generates original samples that preserve the patterns of the real trajectories (e.g., clusters or distance similarity).
For example, in Figure~\ref{fig:teaser}, it can be visually verified that WildGraph generates much better trajectories, than using a Generative Adversarial Network (GAN) for this problem.





\begin{figure}
    \centering
    \begin{subfigure}{0.15\textwidth}
        \centering
        \includegraphics[width=\linewidth, height=1.6cm]{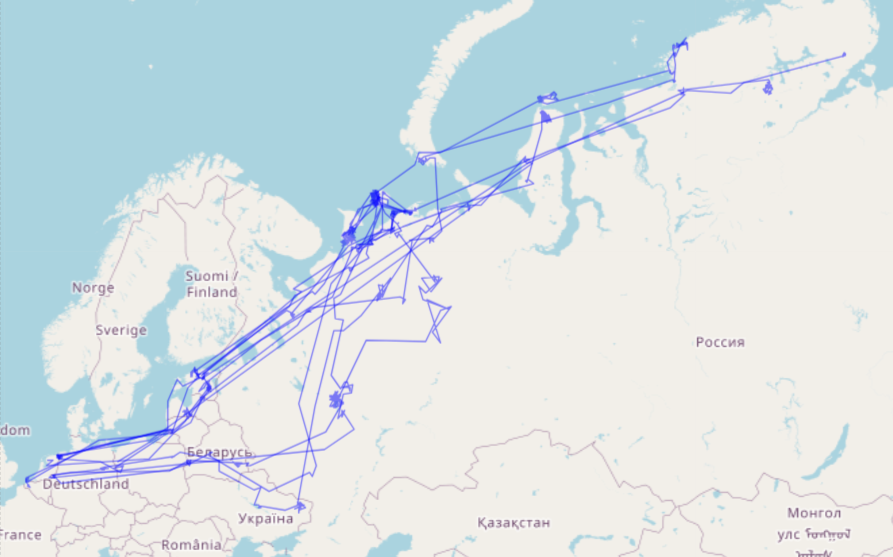}
        \caption{WildGraph}
        \label{fig:image1}
    \end{subfigure}
    \hfill
    \begin{subfigure}{0.15\textwidth}
        \centering
        \includegraphics[width=\linewidth, height=1.6cm]{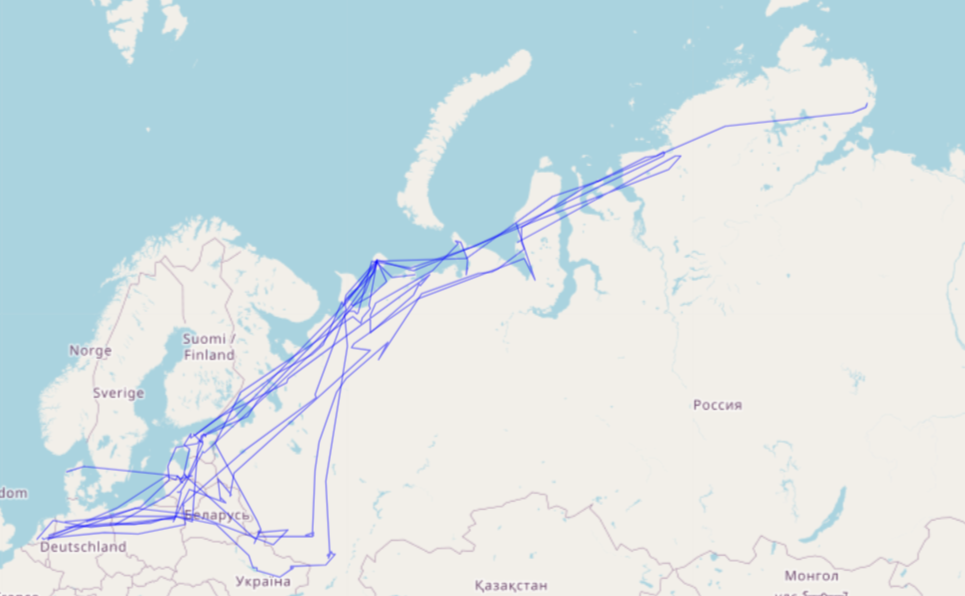}
        \caption{Real}
        \label{fig:image2}
    \end{subfigure}
    \hfill
    \begin{subfigure}{0.15\textwidth}
        \centering
        \includegraphics[width=\linewidth, height=1.6cm]{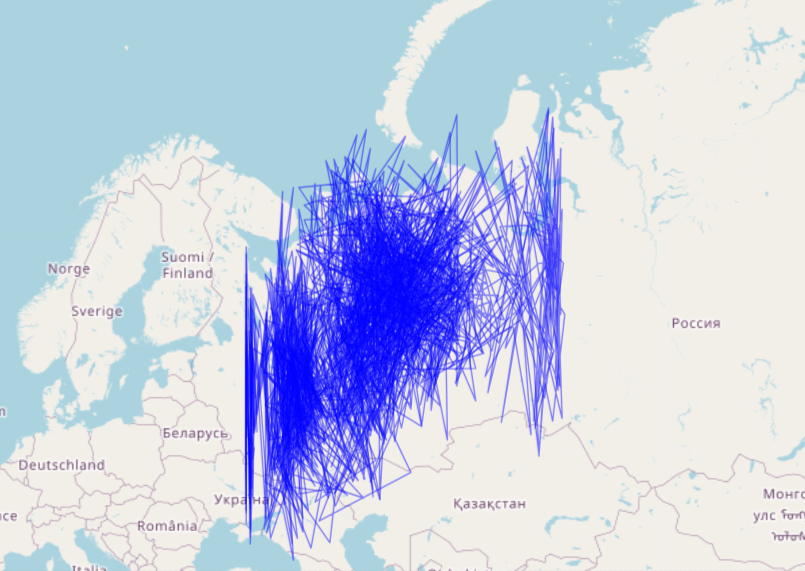}
        \caption{GAN}
        \label{fig:image3}
    \end{subfigure}
    \caption{WildGraph vs. Real vs. GAN-generated Trajectories}
    \label{fig:teaser}
\end{figure}

However, existing solutions in human and vehicle mobility fail to generalize to the wildlife domain. In particular, wildlife trajectory generation presents a unique set of challenges. First is the {\bf small data problem}: real trajectory datasets are sparse and include a small number of subjects. Second, the length of each trajectory tends to span several months, or years (e.g. in that long timespan, they may vary due to confounding factors, weather patterns, extreme events that are outliers, etc.) to provide practical value. Finally, often there is no established network information, e.g. roads, which wildlife may follow, (although in some land animals, there may be tracks on the ground trailed by the migrations, these are not rigid and animals often stray from them to create new trails/paths easily) thus increasing the complexity of the generation task. 

Despite these challenges, interestingly, simple deep learning-based techniques, such as VAEs have demonstrated promising results in generating long-horizon trajectories using a small set of real training samples~\cite{chen_trajvae_2021, al-lawati_wildgen_2023}. VAEs accomplish this by mapping trajectories in a probabilistic manner to a latent space and teaching the decoder to reconstruct the encoded trajectory from the latent representation. Hence, new trajectories that have similar properties as the real trajectory may be generated by sampling the latent space.

However, while samples generated from VAEs appear plausible, they fail to generalize to paths that have not been observed during training.
In addition, point selection along the trajectory by VAEs may not be geographically valid (see Fig~\ref{fig:graph5}). For example, just as vehicles may follow a specific road network, and may stop at rest areas and gas stations, wildlife tend to congregate around water bodies, and vegetation, but not in the middle of the ocean. 



In this paper, we address these challenges by constructing a prototype network of regions for wildlife movement that models trajectories as a path across the network. We utilize the set of known trajectories to
construct the global regions, and recursively refine localized regions, aka nodes, into multiple granular nodes. This helps preserve the global direction of movement, and at the same time captures the local properties of the generated trajectories. By adjusting the region sizes of the network to restrict the selection of potential points from a particular region, we improve the generalization of the generated trajectories, and provide a better outcome across multiple benchmarks. 


Our proposed framework, WildGraph, models trajectories, real or generated, as paths across this network of regions. WildGraph leverages the H3 geographic library, to construct the network. H3 is an indexing system for map regions that has seen increasing use across trajectory-related research~\cite{han_graph-based_2021, musleh_kamel_2023}. It provides a standard, unbiased representation of each region, i.e. the choice of the regions and their sizes are not affected by any given trajectory and does not overfit the training data or a small subset of it. 

Once each known trajectory is encoded as a path over the network, we process it using a Variational Recurrent Network (VRN),
which trains a model that generates a dictionary of latent spaces for each region, as depicted in Figure~\ref{fig:wildgraph}. The details of our framework are discussed in further details in Section 3.

The contributions of this work can be summarized as follows:
\begin{itemize}
    \item We propose WildGraph, a framework that generates realistic long-horizon wildlife trajectories based on a sparse set of real trajectories using a novel VRN implementation.
    \item We contribute a novel approach in prototyping trajectories as paths across a network, and demonstrate its effectiveness in improving inference.
    \item We conduct multiple experiments against two wildlife datasets, and benchmark our methods against our earlier work, WildGEN~\cite{al-lawati_wildgen_2023}, a VAE based approach, a GAN based approach, and a Transformers-based approach based on the prototype network. We use various comparison measures including a geographic imagery-based measure using SatCLIP embeddings~\cite{klemmer2023satclip} to evaluate the geographic similarity between trajectories.
\end{itemize}

To the best of our knowledge, this paper along with our previous workshop paper: WildGEN~\cite{al-lawati_wildgen_2023} are the first to 
propose deep learning-based trajectory generation for wildlife movement. 
Specifically, WildGEN demonstrated that deep learning approaches produce significantly more realistic trajectories than traditional methods, such as Correlated Random Walk (CRW)~\cite{renshaw_correlated_1981} and Levy Flight~\cite{zaburdaev2015levy}. This study builds upon our previous work by refining the point selection process along trajectories, as assessed by the similarity between generated trajectories and real trajectories in our test set.

As a result of the lack of recent application-specific benchmarks, and the key differences of trajectory generation applications in other mobility areas (e.g. based on a network), we compare our method with multiple standard generative algorithms, including VAE, Transformer, and GANs.


\begin{figure}
  \subcaptionbox*{Coarser zoom level}[.23\textwidth]{%
    \includegraphics[width=\linewidth,height=2.5cm]{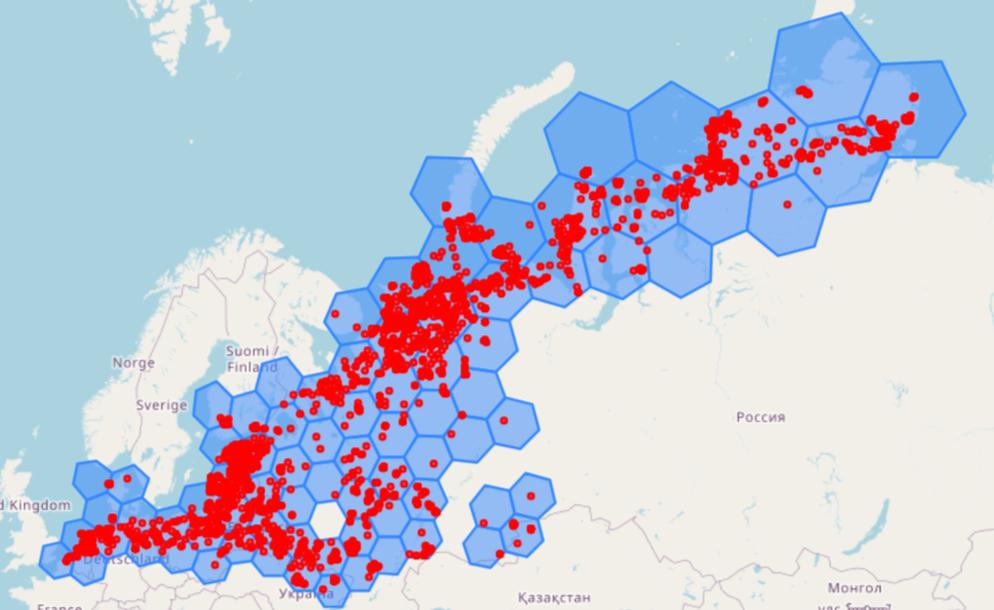}%
  }%
  \hfill
  \subcaptionbox*{Finer zoom level}[.23\textwidth]{%
    \includegraphics[width=\linewidth,height=2.5cm]{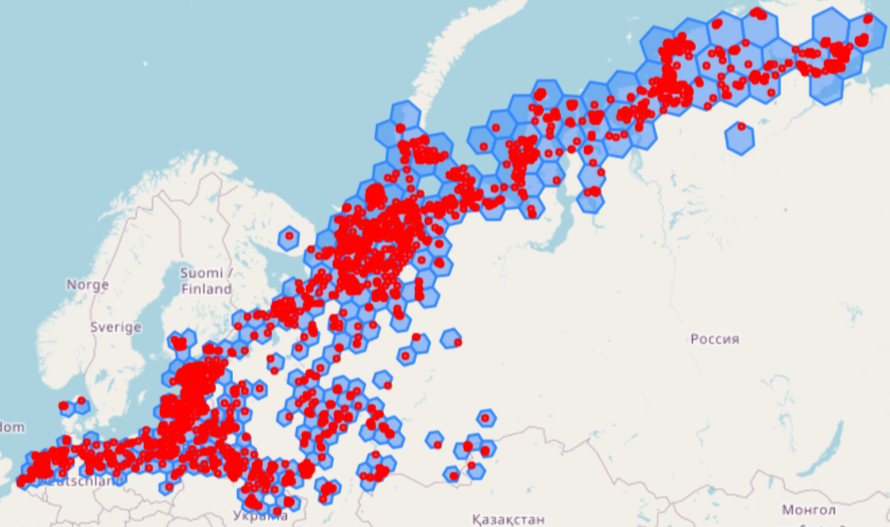}%
  }%

  \caption{
  H3 Regions of different zoom levels covering the set of known points
}
\label{fig:h3}
\end{figure}




\section{Related Work}
In this section, we review recent works from different domains related to the methods and solutions presented in our paper.

\textbf{Trajectory Generation.}
In the pre-deep-learning era, trajectory generation was relatively subject specific. 
In ecology, researchers used methods such as Levy Flight~\cite{zaburdaev2015levy}, Correlated Random Walk (CRW)~\cite{technitis_b_2015}, and other velocity based methods~\cite{technitis_b_2015, ahearn2017context} that rely predominantly on the subject (animal) considered. The advent of deep-learning-based generative models such as VAEs and GANs resulted in some of the earliest attempts at using these approaches for trajectory generation. Chen et al.~\cite{chen_trajvae_2021} benchmarked GANs and VAEs for the task of generation, and found that VAEs provide more consistent results. 
Our earlier work, WildGEN, was among the first methods to address the problem of deep-learning-based trajectory generation for wildlife~\cite{al-lawati_wildgen_2023}.
We utilized VAEs with added layers of post-processing that mitigated some of the challenges discussed above. However, much like other VAE-based methods, it suffers from limited generalization. 

Most recently, Neural Temporal Processes were proposed for the task of generation in human mobility~\cite{wan_generative_2023}. This approach overcomes challenges in dealing with irregularly structured data, which may have various applications in wildlife. However, in our case, the datasets comprise of regularly sampled data.

\textbf{Trajectory Similarity.}
Trajectory similarity is a complementary problem to trajectory generation since the goal of generation is to produce a set of trajectories similar to a set of known trajectories. An interesting approach to the problem was considered by Yao, et al.~\cite{yao_trajgat_2022}, where a region is recursively subdivided to capture the region-based overlap similarity between compared trajectories. Similar to our approach,
they recursively split regions into sub-regions to create a balanced graph. Next, they utilize a graph-based attention layer to model the embedding of each node. 
In comparison, our approach utilizes node2vec~\cite{grover_node2vec_2016} to overcome
the over-parameterization in our problem given the shape of the training data. Furthermore, they represent the trajectory using an adjacency matrix which is not permutation invariant. A time-aware approach to generate similar graph snapshots at each step for a trajectory generation problem would render the approach considerably inefficient. Moreover, it will fail in the presence of cycles in the graph, as it cannot correctly determine the sequence of nodes in the trajectory.

\textbf{Graph Generation.}
Early generative graph models were based on adjacency matrix generation using encoder-decoder structures~\cite{you_graphrnn_2018, kurkova_graphvae_2018}. However, these approaches fail to capture the dynamic nature of spatio-temporal graphs. Recently, models such as STGD-VAEs~\cite{du_disentangled_2022}, and STGEN~\cite{ling_stgen_2023} were proposed to address these limitation. STGD-VAEs factor graphs into dependent ordinal values and jointly optimize the objective. On the other hand, STGEN utilizes the distribution of spatiotemporal walks to adversarially learn the characteristics of the graph. However, multi-objective approaches appear to be a poor fit for this problem, given the small sample sizes, and long horizon. We validate this in our experiments and demonstrate the GANs provide poor results on all metrics (see Fig~\ref{fig:teaser}).

\textbf{Sequence Modeling.}
Sequence-to-Sequence models have been used for various tasks in wildlife trajectory studies, including prediction~\cite{feng_latent_2019}, imputation~\cite{li_prediction_2021} and interpolation~\cite{wan_generative_2023}. They identify implicit patterns by using LSTM or RNN memory trained with new observations in an auto-regressive manner~\cite{park_sequence--sequence_2018}. However, the multivariate nature of spatial data often results in a quick deterioration of predictions on tasks involving long horizon prediction~\cite{sagheer_unsupervised_2019}, particularly in small data settings. To address this, 
Li et al.~\cite{li_prediction_2021} attempted to simplify the problem to a prediction of step size instead of location, but this approach offers very limited usefulness.
Our work adopts a more practical approach by leveraging the properties of the H3 spatial library to represent geospatial regions,
where each region is encoded with a string representing its address
encoded within the library, thus reducing the problem of generating a path to one in a single dimension.

In addition, encoding each region's connection (i.e., the probability of whether a bird will fly from one region to another) with other regions provides additional context to the problem, and 
facilitates generating long-horizon trajectories
based on a small set of training samples.

\section{Preliminaries}
A good generated trajectory accurately reflects the natural movement patterns and behaviors of the wildlife. A key measure is how well the model can predict trajectories it has not encountered, i.e., a held-out set. 

\subsection{Problem Statement}


Formally, given a set of $n$ real trajectories in an area $\mathcal{A}$, each consisting of $m$ points, the objective is to produce a set of synthetic trajectories that maximizes the following:

\begin{enumerate}
    \item (Path Similarity) similarity between the points of the synthetic trajectories to that of the real trajectories.
    \item (Likeness) the alignment of the distribution of the synthetic trajectories with that of a (test) set of real trajectories.
    \item (Coverage) The percentage of samples in the held-out set closest to any generated sample, where ``closest" is defined using a distance metric (we leverage our path similarity metrics). This could be interpreted as a measure of how well the generated samples represent the true distribution of the held-out set. This measure is particularly critical in our context, as it emphasizes the feasibility and representativeness of the generated trajectories.
\end{enumerate}





 The task is to learn a model to generate synthetic trajectories that maximize the similarity between a test set comprising of $k$ real trajectories $T_{test} = T_1, T_2, \ldots, T_k$, each comprising $m$ spatio-temporal points, where $T_i = (p_1^i, p_2^i, \ldots, p_m^i)$, with $p_j^i$ representing latitude ($y$) and longitude ($x$) values. The objective is to find a corresponding set $S = {S_1, S_2, \ldots, S_t}$ of synthetic trajectories that best matches the (test set of) real trajectories.


To achieve this, we aim to maximize the overall similarity between $T_{test}$ and $S$, denoted as: $$\max \text{ Similarity}(T_{test},S)$$

Here, the similarity measure encompasses both Path similarity, Likeness, and Coverage. Note that using a held-out set to compare with the generated trajectories, post-training, ensures robustness and generalization of the proposed solution. Further details on the evaluation metrics utilized are provided in Section 3.5.

\subsection{Overview of WildGraph}
Figure~\ref{fig:wildgraph} demonstrates the overall structure of WildGraph. The training phase proceeds in the following way:
\begin{itemize}
    \item WildGraph accepts a sequence of training samples of known trajectories, each of which consists of $m$ points. All of these points are used in the Hierarchical Network Generator, which implements a prototype graph ($\mathcal{G}$) based on them. A node represents a region of $\mathcal{G}$, while the edges represent transitions from one node to another as observed in a given trajectory.
    \item A pre-training of the nodes of $\mathcal{G}$ is performed based on the connections identified. We utilize node2vec to generate representations of each node based on the input and output connections. The pre-training allows us to encode the relationships between the nodes and provides a representation learning of the transitions in the network, as well as reduces the dimensions of the embeddings, compared to structures such as Bag of words (BoW).
    \item The VRN aims to learn the empirical probability distribution of all states, i.e. learning the transition kernel, in order to generate trajectories following this distribution in an auto-regressive manner. Post-training, the output distribution of each node is aggregated in a latent dictionary (described below) that is used later during the generation process.
\end{itemize}

Once training is completed, the process of generating new trajectories utilizes 
the latent dictionary to stochastically select the next region based on the current region. The latent dictionary is a multi-valued dictionary that
collects latent values ($z$) that correspond to a given state. The generation process is repeated until the desired sequence length (input parameter) is reached. 

For each region, an occupancy based sampling converts the region into a point. The occupancy sampler  divides the entire area $\mathcal{A}$ into smaller regions, i.e. dots, and builds a discrete heatmap on the entire area. The process of converting a region to a point involves sampling a dot (a small region) from the corresponding region based on the heatmap, and randomly sampling a point from the dot.

\subsection{H3 Grid System}
 As mentioned above, we utilize the H3 geospatial indexing system. H3 is an open-source library for creating and managing hierarchical spatial indexes. The core concept of H3 is the use  of hexagonal grids that partition the map into smaller spatial units at different zoom levels. The developers of H3 cite various advantages that include: (1) capacity to model large spatial regions and have better descriptive ability of spherical surfaces; (2) provides a single string-based address (e.g. \textbf{821fa7fffffffff}) that encompasses a geographic region of points; (3) built-in hierarchical description of a region at 14 zoom levels~\cite{h3geo}. 

 Note that H3 divides up the earth's surface into a set of hexagonal regions that are fixed once the zoom level of the hexagon is fixed; however one can choose smaller-sized hexagons to zoom into a region if so required.
 
 Figure~\ref{fig:h3} demonstrates points in the same $\mathcal{A}$ covered by regions using different zoom levels of H3. As mentioned above, each hexagon in either image has a unique string-based address. Our framework utilizes several of the features and benefits provided by H3.

\section{Methodology}

\begin{figure}
\centering
\begin{subfigure}{0.45\textwidth}
    \centering
    \includegraphics[width=\textwidth]{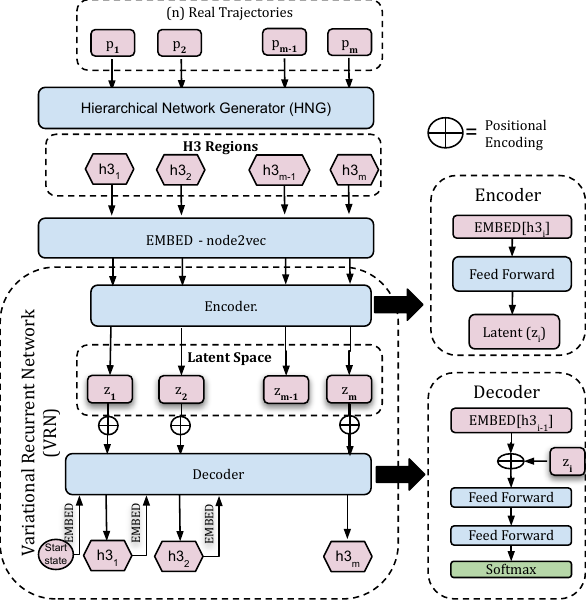}
    \caption{Training}
    \label{fig:subfig1}
\end{subfigure}

\begin{subfigure}{0.45\textwidth}
    \centering
    \includegraphics[width=\textwidth]{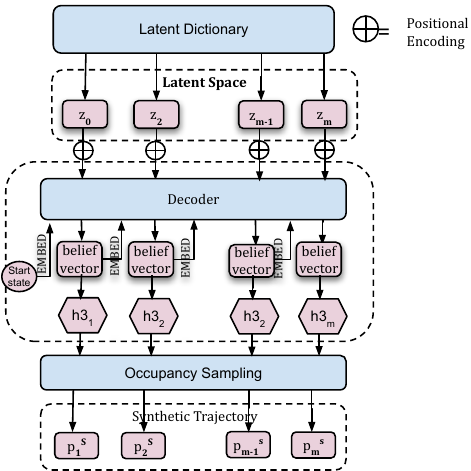}
    \caption{Generation}
    \label{fig:subfig2}
\end{subfigure}

\caption{WildGraph Framework, including both the training and generation steps. Some details are not depicted for clarity. }
\label{fig:wildgraph}
\end{figure}
Figure~\ref{fig:wildgraph} depicts the general architecture of WildGraph. The framework is composed of five main components:

\begin{itemize}
    \item Hierarchical prototype Network Generator (HNG)
    \item Graph-based Embedding Layer
    \item Variational Recurrent Network (VRN)
    \item Latent Dictionary 
    \item Occupancy Sampler
    
\end{itemize}

We describe each component in further detail in following subsections.


\subsection{Hierarchical 
Network Generator}
\begin{figure*}[!htb]
  \subcaptionbox*{Regions divided into sub-regions}[.33\textwidth]{%
    \includegraphics[width=\linewidth]{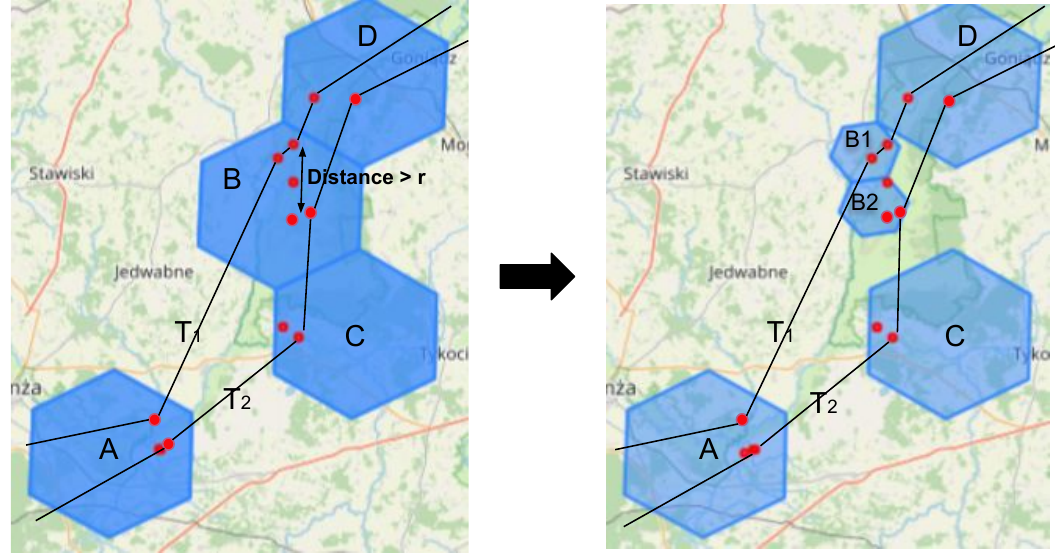}%
  }%
  \hfill
  \subcaptionbox*{A graph of transitions is captured and used to train a graph embedder}[.6\textwidth]{%
    \includegraphics[width=\linewidth]{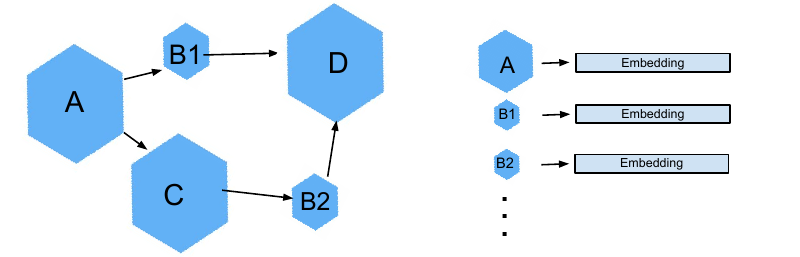}%
  }%

  \caption{
  Hierarchical Network Generator. The red points are the observed data points of where the animal was whereas the black lines are the linear interpolations connecting them.
}
\label{fig:toknizer}
\end{figure*}

The Hierarchical Network Generator (HNG) takes the set of known trajectory points and generates a prototype network graph ($\mathcal{G}$) that represents $\mathcal{A}$. Initially, a low zoom level is used to discretize $\mathcal{A}$, which promotes generalization. However, for each region that encompasses more than a single point, the diameter of the region is calculated and used as basis to discretize it into a higher zoom level as depicted in Figure~\ref{fig:toknizer}. This processing of localizing regions helps preclude unlikely transitions in a generated trajectory. 

The diameter is defined as the maximum distance between two points in a given region. Given a set of points $(p_1, p_2, ...., p_n)$ in a given region ($h3$), the diameter is:
$$
diameter(h3) = \max  || p_i - p_j || \hspace{.5cm} \forall p_i,p_j \in h3
$$
where || · || denotes the Euclidean norm (or distance). In the event $diameter(h3) > r$, for some split threshold parameter ($r$), the region is split, and all points $(p_1, p_2, ...., p_n) \in h3$ are reassigned more granular H3 regions. This process is repeated recursively as captured in Algorithm~\ref{alg:tokenizer}. 

Note that excessive zooming results in a large network that consists of too many nodes, which may limit the generalization of the model, and increase its complexity. The optimal threshold $r$ depends on properties of the dataset, and may be influenced by the resulting average density of the graph. In the experimental section, we empirically investigate how the choice of $r$ affects the evaluation metrics.

Next $\mathcal{G}$ is constructed for observed transitions between regions as depicted in Figure~\ref{fig:toknizer}
\begin{algorithm}[!th]
    \caption{Regionalize Trajectories($T_{train}$, r)}
    \label{alg:tokenizer}
    \begin{algorithmic}[1]
        \STATE  $zoom \leftarrow 2$  \COMMENT{Initialize H3 regions at this zoom level}
        \STATE $P \leftarrow Get\_Points(T_{train})$ \COMMENT{Get set of train-set trajectories}
        \STATE $H \leftarrow Assign\_H3(P, zoom)$  \COMMENT{Assign H3 to each point}
        
        \STATE $H3 \leftarrow SET(H)$  \COMMENT{Get unique H3 regions}
        \FOR{$h3 \in H3$}
            \IF{$diameter(h3) > r$}
                \STATE $zoom \leftarrow zoom + 1$
                \STATE $P_s \leftarrow P[h3]$ \COMMENT{Filter points from $P$ in $h3$}
                \STATE $H_s \leftarrow Assign\_H3(P_s, zoom)$
                \STATE $H3 \leftarrow H3 \setminus \{h3\}$ \COMMENT{Remove $h3$ from $H3$}
                \STATE $H3 \leftarrow H3 \cup H_s$ \COMMENT{Add new $H_s$ to $H3$}
            \ENDIF
        \ENDFOR

        \RETURN  $H3$
    \end{algorithmic}
\end{algorithm}

The final step involves creating the edge list of the network. Algorithm~\ref{alg:tokenizer} regionalizes each trajectory from a sequence of points to a sequence of regions:
\begin{align*}
T_1 = (h3_1^1, h3_2^2, \ldots) \\
T_2 = (h3_2^1, h3_2^2, \ldots) \\
\vdots \phantom{(h3_2^1, h3_2^2, \ldots)}
\end{align*}

The edges of $\mathcal{G}$ are the set of all $(h3_i^k, h3_j^k)$ pairs in each trajectory. While $\mathcal{G}$ is a directed graph, no weights are assigned to the edges.

\subsection{Network Graph Embedding}
A pretraining of $\mathcal{G}$ is performed to learn its characteristics and create embeddings for each node. For simplicity, we use a standard graph-based skip-gram embedding algorithm, namely node2vec to generate the embedding for each state. Node2vec is network representation learning technique to generate low-dimensional vector representations, i.e. embeddings, for the nodes in a graph. It performs a random walk, enabling it to efficiently capture both local and global structural information within a graph~\cite{grover_node2vec_2016}.

Node2vec employs a two-step random walk procedure to explore the graph space. First, it performs short random walks to capture local neighborhood information around each node. These walks balance between breadth-first and depth-first exploration, controlled by two parameters, $p$ and $q$, which dictate the likelihood of revisiting recent nodes or exploring new ones.

Given, the long-horizon nature of our trajectories, we select a higher value of $q$ to explore more distant nodes.
The resulting embeddings learned by node2Vec encode rich structural information about nodes. Node embeddings also significantly reduce the dimensionality of the problem and facilitate off-policy exploration: transitions that are not observed in the real dataset, but are potentially realistic transitions.

\subsection{Variational Recurrent Network}
The VRN uses the embedded sequences of trajectories to train the model by processing them sequentially. It combines elements of recurrent neural networks (RNNs) with variational inference techniques to model the sequential decision-making under uncertainty. Analogous to conventional RNNs, the VRN incorporates recurrent connections that allow it to maintain a hidden state representing the network's memory of past observations and actions. This hidden state evolves over time as the network processes the input sequentially.

The VRN is trained using a combination of self-supervised learning and recurrence techniques. Self-supervised learning is used to train the network to predict the next observation or action given the current observation and hidden state using an encoder-decoder model. Given a pair of region-mapped and graph embedded points $(h3_{emb_i}, h3_{emb_{i+1}})$, the encoder maps $h3_{emb_{i+1}}$ to the latent space representation $z_{i+1}$. The decoder attempts to reconstruct the input vector $h3_{emb_{i+1}}$ from a concatenation of $z_{i+1}$ with $h3_{emb_i}$, modified with positional encoding. The converged $z_{i+1}$ after training is captured and added to the multi-dictionary of $h3_{i}$. 

\subsubsection{Model Formulation}

For each transition $(h3_i,h3_{i+1})$ observed in any given trajectory in $\mathcal{G}$, The input embedding \( h3_{emb_i} \) is computed as:
\[
h3_{emb_i} = \text{Embedding}(h3_i)
\]

The latent representation \( z_i \) is obtained using a MLP encoder ($\theta$'s are learnt weights):

\[
z_{i+1} = \text{MLP}_{\text{enc}}(h3_{emb_{i+1}}; \theta_{\text{enc}})
\]

The decoder's hidden state \( h_{\text{dec}} \) is computed using a MLP decoder:
\[
h_{\text{dec}} = \text{MLP}_{\text{dec}}(h3_{emb_i}, z_{i+1}, PE; \theta_{\text{dec}})
\]

The logits are computed using a linear transformation of the decoder's output:
\[
\mathbf{o} = \text{Linear}(\mathbf{h}_{\text{dec}}, \mathbf{w}_{\text{out}}, \mathbf{b}_{\text{out}})
\]
where \(\mathbf{w}_{\text{out}}\) is the weight matrix and \(\mathbf{b}_{\text{out}}\) is the bias vector.

The logits are transformed into probabilities using the softmax function:
\[
prob_k = \frac{\exp(o_k)}{\sum_{j=1}^{V_{\text{out}}} \exp(o_j)}
\]
where \( prob_k \) is the predicted probability for the \( k \)-th class and \( V_{\text{out}} \) is the size of the output vocabulary, i.e. the number of regions.

\subsubsection{Loss Function}
The cross-entropy loss between the true label \( h3_{i+1} \) and the predicted probabilities \( \mathbf{prob} \) is computed as:
\[
\mathcal{L} = -\sum_{k=1}^{V_{\text{out}}} enc(h3_{i+1}) \log(prob_k)
\]
where \( enc(h3_{i+1})\) is the one-hot encoded true label.

Post-training, the hidden state (the output of the bottleneck layer in the autoencoder) reflects the possible transitions learned from the training trajectories, though it is represented in a compressed space.

\subsubsection{Generation}
 
During generation, we first sample the latent space, add positional encodings (PE), and append the belief vector, which is a weighted sum of the top-$y$ states based on their probabilities (\(\sum_{i=1}^{y} \text{Prob}(h3_i) \times h3_{emb_i}\)), at the previous time step. This belief vector considers all previously visited states autoregressively, rendering each generated trajectory unique. Empirical results demonstrate that even a set of 1000 generated trajectories exhibits a unique ordering of regions. We formally prove the low likelihood of duplicate generated trajectories in Appendix~\ref{appendixa}.


\subsubsection{Positional Encoding}

Positional encoding (PE) assigns unique representations to each position in the sequence. This allows the model to differentiate between regions based on their position in the sequence. For example, a bird may remain in a given region for several days before proceeding to the next region. PE enables the model to learn from the sequential context effectively. We use a simple approach by concatenating the raw index of the position to the input of the decoder. This helps the network to associate certain transitions with the position, while avoiding the risk that the VRN may fail to learn a more sophisticated method given the small data nature of the problem.

\subsection{Latent Dictionary}
After training, the representation of each possible state $h3$ is passed through the encoder to obtain the corresponding latent space realization $z$. Using the frequentist approach, for each node $h3$, the possible transitions are collected in a dictionary and then replaced with their corresponding latent space realizations. When a certain node appears as the current state in our VRN, we sample uniformly from the latent dictionary of this node. This is akin to sampling the stationary next state based on the possible transitions in the training set, but in latent space. However, these stationary transitions are combined with the belief vector, thus learning the non-stationary transitions, while taking into consideration all previously generated states. As discussed, this belief vector is computed based on the $y$ most probable $h3$ regions and their corresponding node2vec embeddings.

Formally, the dictionary $D$ can be represented as a mapping for all nodes $(k_1, \ldots, k_s)$ in $\mathcal{G}$:
$$\{k_1, \ldots, k_s\}: \{(l_1^{j_1}), \ldots, (l_s^{j_s})\} 
$$
For all pairs of transitions between nodes in the training set: $(k_i, k_j)$, the dictionary is updated: $D' = D \cup \{k_{i}: (l_j)\}$ where $l_j = MLP_{enc}(k_j)$.

\subsection{Occupancy Sampler}
The occupancy sampler chooses a dot (a very small region) and then from the dot, it chooses a point at random.
The sampling is conducted based on the heatmap of the region as depicted in Figure~\ref{fig:occupancy}. The process is as follows:
\begin{enumerate}
    \item Divide the region into dots, where each dot represents an H3 region with a high zoom level that lies within the region.
    \item Based on the distribution of real points inside the regions, sample a dot.
    \item Sample a point that lies inside the dot randomly.
\end{enumerate}
The intuition behind using a two-step hierarchical sampler (i.e., going to dots and then choosing points) as opposed to choosing points directly from the heatmap is as follows.
If we represented a region as a set of points and used a heatmap to choose the points directly, then we would be choosing points from existing trajectories with high probability. By introducing a small region that we select based on the heatmap and randomizing the choice of the point from the dot, we introduce a level of generalization that is useful in introducing additional stochasticity in point selection.
An added advantage of this approach is it helps ensure the points selected have geographic validity, i.e., are close enough to actual points because of how we choose the output points based on a dot, which is a very small region. This can been visually confirmed from the generated paths in Figure~\ref{fig:graph1}, where WildGraph trajectories, much like real trajectories, almost always avoid 
congregating in the middle of the sea. In contrast, this is not the case with some benchmark methods, including competitive methods such as WildGEN.

\begin{figure}[t]
    \centering
    \begin{subfigure}{.27\textwidth}
        \centering
        \includegraphics[width=\linewidth]{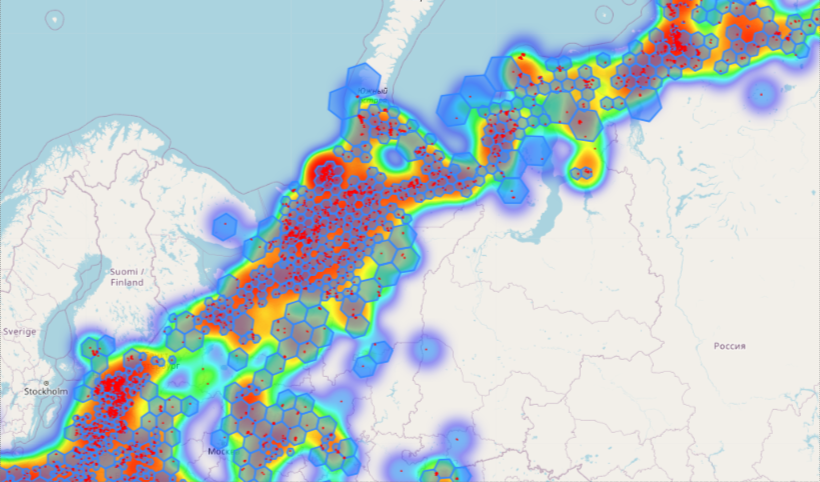}
        \caption{Regions of different sizes can be observed}
        \label{fig:image1}
    \end{subfigure}

    \begin{subfigure}{.27\textwidth}
        \centering
        \includegraphics[width=\linewidth]{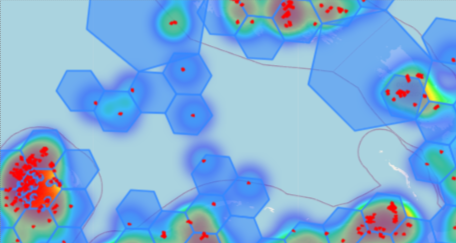}
        \caption{Zooming shows a sample distribution of dots on different regions}
        \label{fig:image2}
    \end{subfigure}

    \begin{subfigure}{.27\textwidth}
        \centering
        \includegraphics[width=\linewidth]{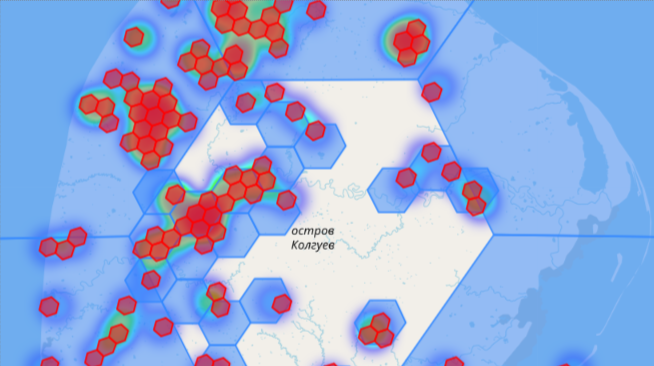}
        \caption{Further zooming shows dots as small regions}
        \label{fig:image3}
    \end{subfigure}
    \caption{The breakdown of the spatial domain into regions (blue), and dots (red) projected over a heatmap}
    \label{fig:occupancy}
\end{figure}


\section{Experiments}

\subsection{Baselines}
In order to assess the performance of our framework, we compare the results obtained to
those from existing methods for animal trajectory generation. As discussed earlier, CRW-based and other traditional motion methods fail at generating long-horizon trajectories for tasks such as migration. This was demonstrated in our workshop paper~\cite{al-lawati_wildgen_2023}.

As such we restrict our evaluation to deep-learning based baselines that appear to be plausible for the task of long-horizon trajectory generation, namely:

\begin{itemize}
    \item VAE
    \item WildGEN
    \item GAN
    \item Transformer
    \item Levy Flight
\end{itemize}

In the following subsections, we briefly introduce each method.

\subsubsection{VAEs} As discussed earlier, VAEs are a supervised training model that map inputs into a latent space while minimizing the distributional divergence between the input and the output. The latent space is subsequently sampled to generate original samples.

\subsubsection{WildGEN} WildGEN is a model based on VAEs that performs post-processing on the generated trajectories. Two forms of post-processing are applied:
\begin{itemize}
    \item Smoothing: utilizes a Savitzky-Golay smoothing filter~\cite{press1990savitzky} to reduce excessive wandering.
    \item Minimum Bounding Region: drops any trajectories that venture out of a minimum bounding polygon of the known points.
\end{itemize}

\subsubsection{Generative Adversarial Network (GAN)}

GAN is a model consisting of two neural networks, the generator and the discriminator, which engage in a minimax game. The generator aims to produce realistic synthetic data samples from random noise, while the discriminator learns to distinguish between real and generated samples. Through adversarial training, GANs progressively improve both the generator's ability to create realistic data and the discriminator's proficiency in discerning real from fake. This adversarial process drives the generator to generate increasingly realistic samples~\cite{goodfellow_generative_2020}. 

\subsubsection{Transformer}
Transformers are networks that have the capacity to auto-regressively generate sequences of tokens by attending to previous tokens. The HNG facilitates the use of H3 word-based language transformers, such that a trajectory is a sentence that consists of $m$ words, each of which is an H3 address.

\subsection{Levy Flight}
Levy Flight~\cite{zaburdaev2015levy} is a stochastic process that models movement using a Cauchy distribution to generate trajectory points. These points are determined by step length, variance, and angular standard deviation, which we measure using the real trajectory samples. The algorithm utilizes these parameters to generate trajectories that attempt to capture and reproduce the movement patterns.

\subsection{Evaluation Metrics}
Our selection of evaluation metrics is based on the objectives of trajectory generation, which we reiterate here for convenience:

\begin{enumerate}
    \item (Path Similarity) How similar are the points of the synthetic trajectories to the real trajectories?
    \item (Likeness) How well do the synthetic trajectories align with the real trajectories in terms of areas of concentration?
        \item (Coverage) The percentage of samples in the held-out set that are closest to any generated sample, for each path similarity metric.
\end{enumerate}

\begin{table*}[h!]
   \caption{Test Results for the geese and stork datasets using WildGraph and other benchmarks using different measurements} 
   \label{tab:results}

   \centering
   \scalebox{0.9}{
   \begin{tabularx}{0.86\textwidth}{l|l||cc|cc|cc|cc||c|c}
   \toprule\toprule
   \multirow{2}{*}{\textbf{Dataset}} & 
   \multirow{2}{*}{\textbf{Method}}  & 
   \multicolumn{2}{c|}{\textbf{Hausdorff}} &  
   \multicolumn{2}{c|}{\textbf{DTW}} & 
   \multicolumn{2}{c||}{\textbf{FDE}}  & 
   \multicolumn{2}{c||}{\textbf{SatCLIP}}  & 
   \multirow{2}{*}{\textbf{r-Coefficient}}  & 
   \multirow{2}{*}{\textbf{$chi^2$}} \\ 
    & &value       &cov  &value       &cov  &value &cov   &value &cov    \\
   \midrule
   \multirow{5}{*}{\textbf{Geese} } 
   & VAE     & 7.38 & 0.56 & 56.85 & 0.66 & 3.67 & 0.82 & 0.57 & 0.53 & 0.47 & 1672  \\
   & GAN     & 25.53 & 0.33 & 194.74 & 0.25 & 8.62 & 0.4 & 0.44 & 0.31 & 0.30 & 7490\\
   & WildGEN  & \textbf{7.15} & 0.71 & \textbf{48.94} & 0.82 & 3.24 & \textbf{0.86} & \textbf{0.6} & 0.8 & \textbf{0.80} & 695\\
   & 
   \textbf{WildGraph}    &  9.63 & \textbf{0.93} & 71.08 & \textbf{0.92} & \textbf{2.93} & 0.81 & 0.57 & \textbf{0.89} & \textbf{0.80} & \textbf{628}\\
   & Transformer    & 20.82 & 0.88 & 216.08 & 0.78 & 13.58 & 0.78 & 0.40 & 0.67 & 0.48 & 1636\\
      & Levy Flight & 31.96 & 0.09 & 332.14 & 0.12 & 29.29 & 0.08 & 0.26 & 0.29 & -0.09 & 9706 \\

   \bottomrule
      \multirow{5}{*}{\textbf{Stork} } 
   & VAE     & 8 & 0.71 & 43.3 & 0.68 & 4.22 & 0.65 & 0.58 & 0.74 & 0.56 & 666\\
   & GAN     & 14.77 & 0.23 & 59.95 & 0.21 & 3.9 & 0.42 & 0.56 & 0.35 & 0.62 & 1491\\
   & WildGEN  & \textbf{5.66} & 0.68 & \textbf{29.26} & 0.68 & \textbf{3.22} & 0.68 & \textbf{0.63} & 0.69 & 0.77 & 743\\
   & \textbf{WildGraph}    &  6.61 & \textbf{0.84} & 35.98 &\textbf{ 0.81} & 3.6 &\textbf{ 0.74} & 0.59 &\textbf{ 0.83} & \textbf{0.78} & \textbf{438}\\
   & Transformer    &  16.23 & 0.65 & 87.19 & 0.66 & 8.95 & 0.58 & 0.53 & 0.65 & 0.56 & 642\\
      & Levy Flight & 25.75 & 0.21 & 202.21 & 0.09 & 31.27 & 0.11 & 0.37 & 0.17 & 0.09 & 4492\\ 

   \bottomrule
   \end{tabularx}
   }
\end{table*}

In order to measure Path Similarity, we use three distance measures: Hausdorff Distance, which measures the steepest gap, Dynamic Time Warping (DTW), which measures the spatial similarity over the whole path, and Final Displacement Error (FDE), which we define as the gap between the final positions of a generated trajectory and the nearest real trajectory: $\text{FDE} = \|p_t^m - \hat{p}_s^m\|$. 

In addition, we also utilize a satellite imagery-based method, SatCLIP, to measure feature similarity. SatCLIP generates a spatial embedding that summarizes the geographic representation of trajectories, which can then be compared using cosine similarity. The choice of SatCLIP is motivated by its capacity to embed remote geo-locations where other embedding methods, such as hex2vec~\cite{wozniak_hex2vec_2021} were found to provide limited or no data. 

On the other hand, we measure Likeness using Pearson Correlation Coefficient (r-Coeff) and Chi-squared ($chi^2$) similarity.

For each generated trajectory, we compare it to every trajectory from the test set and use the closest match. The reported distance metric for the experiment is the mean of these values over a fixed number of generated trajectories.

Finally, coverage (cov) measures how representative the generated trajectories are of the test set of real trajectories. This measure indicates what percentage of trajectories in the held-out-set were closest to a generated trajectory. For example, given generated trajectory $S_i$, if the best Hausdorff or DTW measure was observed when compared to a real trajectory $T_j$, we add $T_j$ in the coverage list. This measure quantifies the fraction of trajectories added in the coverage list for each distance metric.

\subsection{Datasets}

\begin{table}[h!]
    \centering
    \caption{Summary of Geese and Stork Datasets}
    \resizebox{0.5\textwidth}{!}{
    \begin{tabular}{@{}lccc@{}}
        \toprule
        \textbf{Dataset} & \textbf{Clipping Window} & \textbf{Length (days)} & \textbf{Sample Size} \\ 
        \midrule
        Western Palearctic greater white-fronted geese~\cite{geese} (\textbf{Geese}) & Mar 1 - Sep 1 & 185 & 60 \\ 
        White Stork Ciconia ciconia~\cite{stork} (\textbf{Stork}) & Aug 15 - Nov 1 & 79 & 77 \\ 
        \bottomrule
    \end{tabular}
    }
    
    \label{tab:datasets}
\end{table}

Table~\ref{tab:datasets} lists the two datasets of migratory birds utilized for testing, Both datasets were obtained from Movebank\footnote{https://www.movebank.org} and were processed using Moveapps~\cite{kolzsch2022moveapps} to limit observations to one per day for each subject. For both datasets, any subject
that had significant missing readings was dropped. Subjects that were tracked over multiple years were split into separate subjects. Simple data-interpolation was utilized to fill in any remaining isolated gaps in the trajectories.

\subsection{Experimental Settings}
Preliminary tests were performed to evaluate the convergence properties of each method. We used a fixed learning rate of \(1 \times 10^{-3}\) with the Adam optimizer. For the GAN experiment, additional decay rate parameters were used to induce convergence.

Each experiment used 5-fold cross-validation with a different held-out set to generalize given the small data nature of the problem. All reported values are the average of five identical experiments. This choice was made to normalize the randomness of the generation process and report more consistent results. 

The experiments were conducted utilizing a single 80GB NVIDIA A100 GPU paired with 8 CPUs, operating on Python 3.8, PyTorch 2.1.1, and CUDA 11.8.

\subsubsection{Distributional Settings}
To measure the likeness (r-Coeff and $chi^2$), we utilize the discrete cluster distributions between the generated points and the test set. To create the clusters, the optimal cluster size for K-means was pre-calculated and used as input in the distributional likeness measurements. The cluster sizes were determined using a Silhouette plot and were selected visually. Specifically, the cluster sizes for the geese and stork datasets are 16 and 14, respectively.

\subsubsection{Split Threshold}
As discussed in section 4.1, here, we analyze the effect of the split threshold on generalization. Figure~\ref{fig:split-threshold} depicts how coverage values and r-Coeff are affected by different values of the split threshold.

\begin{figure}[h]
\includegraphics[width=0.3\textwidth, height=3.3cm]{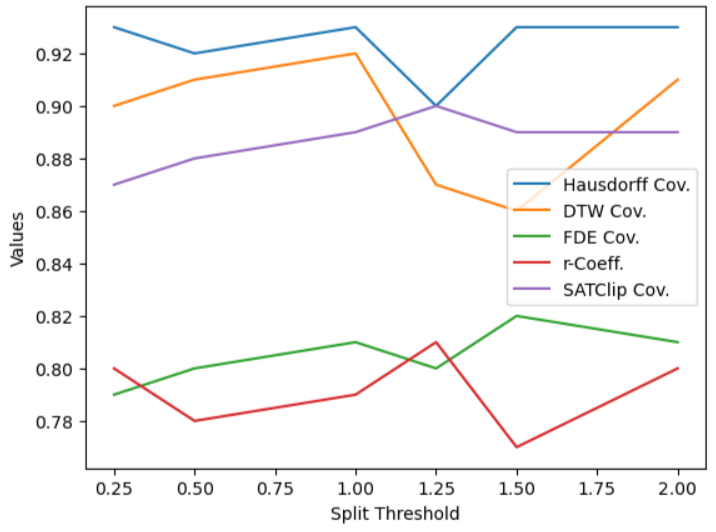} 
\centering
\caption{Coverage as a function of split threshold}
\label{fig:split-threshold}
\end{figure}

For our experiments, we use a split threshold (r) value of 1 which was experimentally validated to provide high coverage, while providing high r-Coeff.

In the ablation study, we explore how the results change as a function of the split threshold.

\subsection{Results}

Table~\ref{tab:results} lists the results obtained for WildGraph and the other benchmarks on the abovementioned metrics. WildGraph achieves superior overall performance on most coverage and likeness measurements. Regarding the distance metrics, WildGraph reports results comparable to those of VAE and WildGEN on both datasets, with some exceptions where WildGEN outperforms WildGraph. However, given that WildGraph provides better coverage results, which implies that the generated trajectories from WildGraph generalize better on the held-out set. On the other hand, WildGEN fails to generalize on a larger subset of the held-out set compared to WildGraph. A key limitation of WildGEN is that it applies to a continuous spatial domain, leading to unrealistic paths, such as birds flying across the ocean. By discretizing the spatial domain, WildGraph ensures generated paths remain within realistic regions. This makes WildGraph more useful for generating trajectories similar to real, unknown paths.

Similarly, while WildGEN provides a higher SatCLIP result, it also fails to generalize as effectively as WildGraph. WildGraph outperforms all other methods on this measure. However, this finding also suggests that while paths traversed by birds can be captured well to some extent using SatCLIP, it may fail to address all the nuances of the movement. To understand this further, we measure the average pairwise similarity between real data as 0.4845 and 0.4805 for the geese and stork datasets, respectively. In the experiments, however, the results considered only the closest matches from the training set against each trajectory in the held-out set, as discussed above, which may explain the higher scores. 

In summary, the results demonstrate WildGraph's exceptional capacity to learn the properties of very small datasets but still generate new trajectories in a meaningful way that exceeds all other methods in terms of generalization while providing comparable distance, likeness, and geographic similarity measures.

\subsection{Ablation Study}
\begin{table*}[tb]
   \caption{Test Results for the geese and stork datasets that demonstrate the contribution of key components and constructs} 
   \label{tab:example}

   \centering
   \scalebox{0.7}{
   \begin{tabularx}{0.9\textwidth}{l|l||cc|cc|cc|cc||c|c}
   \toprule\toprule
   \multirow{2}{*}{\textbf{Dataset}} & 
   \multirow{2}{*}{\textbf{Method}}  & 
   \multicolumn{2}{c|}{\textbf{Hausdorff}} &  
   \multicolumn{2}{c|}{\textbf{DTW}} & 
   \multicolumn{2}{c||}{\textbf{FDE}}  & 
   \multicolumn{2}{c||}{\textbf{SatCLIP}}  & 
   \multirow{2}{*}{\textbf{r-Coefficient}}  & 
   \multirow{2}{*}{\textbf{$chi^2$}}\\ 
    & &value       &cov  &value       &cov  &value &cov   &value &cov    \\
   \midrule
   \multirow{5}{*}{\textbf{Geese} } 
   & Uniform H3 / Coarse     & 10.02 & \textbf{0.93} & 73.68 & \textbf{0.92} & 4.03 & \textbf{0.87} & 0.51 & 0.85 & 0.62 & 815\\
   & Uniform H3 / Granular      & 10.3 & 0.85 & 83.18 & 0.82 & 3.21 & 0.75 & 0.56 & 0.8 & 0.68 & 889\\
   & BoW (no node2vec)  & 15.99 & 0.1 & 108.11 & 0.08 & \textbf{1.67} & 0.17 & 0.54 & 0.1 & 0.43 & 8364\\
     & No PE    & \textbf{8.07} & 0.92 & \textbf{59.63} & 0.9 & 3.02 & 0.82 & 0.53 & \textbf{0.93} & 0.71 & 762\\

   & \textbf{WildGraph}    &  9.63 & \textbf{0.93} & 71.08 & \textbf{0.92} & 2.93 & 0.81 & \textbf{0.57} & 0.89 & \textbf{0.80} & \textbf{628}\\

   \bottomrule
      \multirow{5}{*}{\textbf{Stork} } 
   & Uniform H3 / Coarse     & \textbf{6} & 0.84 & \textbf{32.57} & 0.77 & 3.07 & \textbf{0.75} & 0.58 & 0.81 & 0.7 & \textbf{339}\\
   & Uniform H3 / Granular      & 7.27 & 0.57 & 41.84 & 0.6 & 4.41 & 0.56 & \textbf{0.6} & 0.75 & \textbf{0.78} & 348\\
   & BoW (no node2vec)  &12.36 & 0.13 & 52.39 & 0.08 & \textbf{2.02} & 0.08 & 0.49 & 0.09 & 0.13 & 3932\\
     & No PE    &6.57 & 0.79 & 38.18 & 0.74 & 4.14 & 0.67 & 0.59 & 0.77 & 0.73 & 421\\

   & \textbf{WildGraph}    &   7.11 & \textbf{0.84} & 38.11 & \textbf{0.8} & 3.67 & 0.74 & 0.59 & \textbf{0.85} & 0.74 & 369\\

   \bottomrule
   \end{tabularx}
   }
\end{table*}

We perform two ablation studies: \begin{enumerate}
    \item Measure how each component of WildGraph contributes to the results achieved.
    \item Investigate how different values of the split threshold affect the results.

\end{enumerate}

As evident from Table \ref{tab:example}, node embeddings are the most impactful component of the WildGraph framework based on the observed results. Using BoW, however, provided poor results, possibly due to its high dimensionality and the lack of node relationships in the embeddings. Additionally, PE significantly enhances the coverage of the methods. Dynamically breaking the regions with the HNG significantly improves likeness, while using coarse or granular regions shows limited effectiveness. These results highlight WildGraph's robustness and efficiency across different evaluation metrics when combining all its components.

In addition, by running experiments with fixed region sizes at both high zoom (Granular) and low zoom (Coarse) levels, it can be observed that the HNG is a key component that improves trajectory similarity specifically, but also the likeness on most evaluation metrics.

For additional insight into our solution, we also evaluate how different values of split threshold impact the results. This is summarized in Table \ref{tab:example2}. Subtle variations in the results as a function of the split threshold is observed. The choice of threshold can be motivated by the dataset at hand. While WildGraph (using the specified split-threshold) is not the best in every metric, it's close to the best in both distance and coverage. No method outperforms all others in every metric. On the whole, we believe a threshold of $1.0$  is the best overall and provides the best compromise on these metrics for both datasets.
However, for trajectories where the length is shorter, a smaller threshold may provide a better outcome.

\subsection{Qualitative Analysis}
Figure \ref{fig:six_graphs} depicts the performance of WildGraph in comparison to other benchmarks. A preliminary visual examination reveals that WildGraph trajectories exhibit more realistic paths. Unlike WildGEN and VAEs, and similar to the real samples, trajectories do not cluster in the middle of the sea. It is also observed that GAN and Transformer-based approaches provided poor results, which is expected given the small data nature of the problem. More results including Levy Flight are omitted for brevity and may be examined in~\cite{al-lawati_wildgen_2023}.

\begin{table*}[tb]
   \caption{Test Results for the geese and stork datasets using WildGraph with different split threshold configurations} 
   \label{tab:example2}

   \centering
   \scalebox{0.7}{
   \begin{tabularx}{0.9\textwidth}{l|l||cc|cc|cc|cc||c|c}
   \toprule\toprule
   \multirow{2}{*}{\textbf{Dataset}} & 
   \multirow{2}{*}{\textbf{Split Threshold}}  & 
   \multicolumn{2}{c|}{\textbf{Hausdorff}} &  
   \multicolumn{2}{c|}{\textbf{DTW}} & 
   \multicolumn{2}{c||}{\textbf{FDE}}  & 
   \multicolumn{2}{c||}{\textbf{SatCLIP}}  & 
   \multirow{2}{*}{\textbf{r-Coefficient}}  & 
   \multirow{2}{*}{\textbf{$chi^2$}}\\ 
    & &value       &cov  &value       &cov  &value &cov   &value &cov    \\
   \midrule
   \multirow{5}{*}{\textbf{Geese} } 
   & 0.25     & 9.55 & 0.93 & 71.28 & 0.9 & 2.76 & 0.79 & 0.57 & 0.87 & 0.8 & 614\\
   &0.5      & 9.62 & 0.92 & 70.06 & 0.91 & 2.96 & 0.8 & 0.56 & 0.88 & 0.78 & 669\\
   
   & \textbf{WildGraph (1.0)}    &  9.63 & \textbf{0.93} & 71.08 & \textbf{0.92} & \textbf{2.93} & 0.81 & \textbf{0.57} & 0.89 & 0.80 & \textbf{628}\\
   & 1.5  & \textbf{9.29 }& 0.9 & \textbf{66.5 }& 0.87 & 2.98 & 0.8 & 0.55 & \textbf{0.9} & \textbf{0.81} & 635\\
     & 2.0   & 9.52 & 0.93 & 70.15 & 0.91 & 3.15 & \textbf{0.81 }& 0.54 & 0.89 & 0.8 & 632\\
   \bottomrule
      \multirow{5}{*}{\textbf{Stork} } 
   & 0.25    & 7.14 & 0.77 & 37.6 & 0.75 & 3.44 & 0.65 & \textbf{0.61} & 0.82 & 0.78 & 407\\
   & 0.5      & 6.97 & 0.8 & \textbf{35.88} & 0.75 & \textbf{3.29} & 0.7 & 0.6 & 0.83 & 0.74 & 447\\
      & \textbf{WildGraph (1.0)}      &  \textbf{6.61} & 0.84 & 35.98 & \textbf{0.81} & 3.6 &\textbf{ 0.74 }& 0.59 & 0.83 & \textbf{0.78} & 438\\
   & 1.5 & 7.07 & 0.82 & 37.84 & 0.77 & 3.61 & 0.71 & 0.58 & \textbf{0.86} & 0.74 & 422\\
     & 2.0    & 6.84 & \textbf{0.86} & 36.65 & 0.78 & 3.51 & 0.73 & 0.59 & 0.84 & 0.75 & \textbf{352}\\
   \bottomrule
   \end{tabularx}
   }
\end{table*}

\begin{figure*}[htb]
    \centering
    \begin{subfigure}[b]{0.28\textwidth}
        \centering
        \includegraphics[width=\textwidth, height=2.5cm]{images/wildgraph.png}
        \caption{WildGraph}
        \label{fig:graph1}
    \end{subfigure}
    \hfill
    \begin{subfigure}[b]{0.28\textwidth}
        \centering
        \includegraphics[width=\textwidth, height=2.5cm]{images/real.png}
        \caption{Real Samples}
        \label{fig:graph2}
    \end{subfigure}
    \hfill
    \begin{subfigure}[b]{0.28\textwidth}
        \centering
        \includegraphics[width=\textwidth, height=2.5cm]{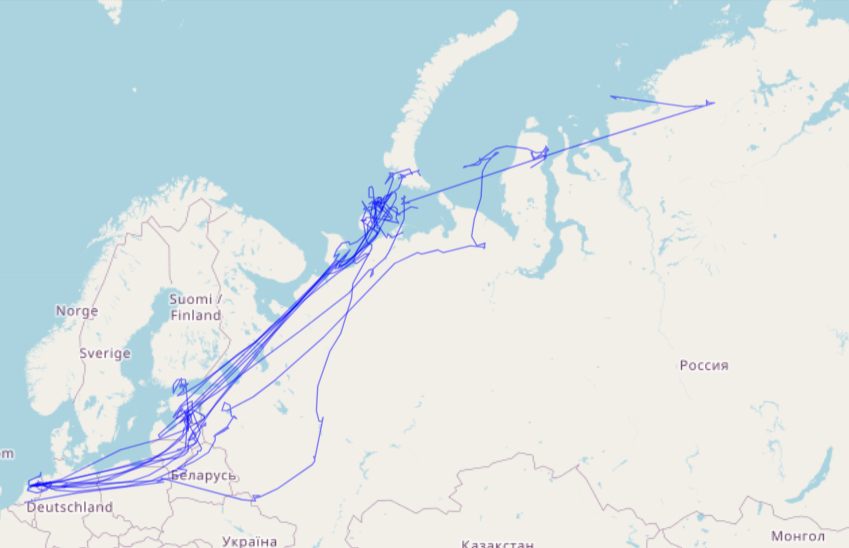}
        \caption{WildGEN}
        \label{fig:graph3}
    \end{subfigure}
    \vskip\baselineskip
    \begin{subfigure}[b]{0.28\textwidth}
        \centering
        \includegraphics[width=\textwidth, height=2.5cm]{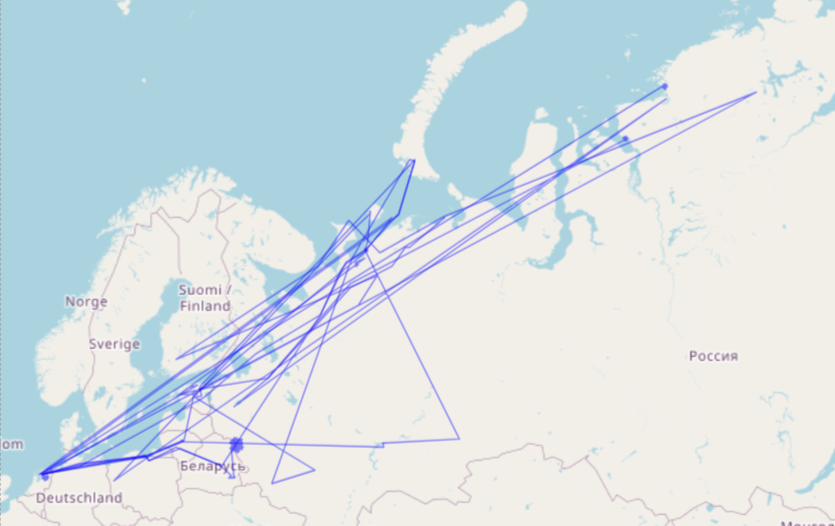}
        \caption{WildGraph-Transformers}
        \label{fig:graph4}
    \end{subfigure}
    \hfill
    \begin{subfigure}[b]{0.28\textwidth}
        \centering
        \includegraphics[width=\textwidth, height=2.5cm]{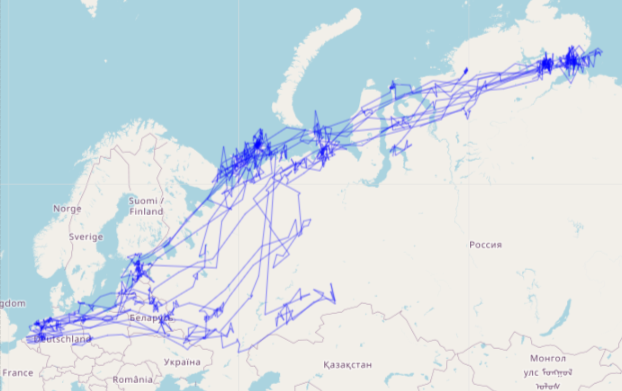}
        \caption{VAE}
        \label{fig:graph5}
    \end{subfigure}
    \hfill
    \begin{subfigure}[b]{0.28\textwidth}
        \centering
        \includegraphics[width=\textwidth, height=2.5cm]{images/gan.png}
        \caption{GAN}
        \label{fig:graph6}
    \end{subfigure}
    \caption{Random Real and Generated Trajectories from WildGraph and the benchmarks}
    \label{fig:six_graphs}
\end{figure*}

\section{Conclusions}
In this paper, we study the problem of trajectory generation over a long horizon, with a focus on applications in wildlife in a small data setting. Existing methods fail to consider the underlying network that describes this movement. In order to solve this problem, we introduced a hierarchical prototype network that captures the global movement characteristics and recursively localizes regions. We proposed a novel VRN to capture the transitions and feed them into a latent dictionary. The effective use of the H3 library, and node2vec embeddings has contributed to the success of the WildGraph framework, as measured in the experimental results and ablation study.

This study adds to the body of literature where graph-based solutions can effectively address unique problems.
Ours is the first to show how one can generate trajectories for wildlife movement. In our case, the small data nature of wildlife datasets combined with wide-horizon predictions presented significant challenges.
Our results demonstrate the effectiveness of our proposed 
approaches that utilize specially designed encoder-decoder network and work remarkably well.

Furthermore, similar to other generative models, such as GPT, our method treats regions as tokens and trajectories as an ordered list of tokens. With increased training data, it may be possible to fine-tune for downstream applications, such as identifying bird characteristics, like age or type, based on their movement patterns.

Finally, it may be interesting to explore how this approach could be adapted for similar problems such as the trajectory prediction task. Additionally, the simplicity of the approach might generalize well to congruous problems in the vehicle movement or human movement domains.

\bibliographystyle{ACM-Reference-Format}
\bibliography{MAIN}

\appendix
\section{Theoretical Analysis on Probability of Replication}\label{appendixa}
Here, we investigate the probability of WildGraph generating a sequence of regions that replicates a real trajectory in the training set. 
First, to summarize WildGraph, at each transition step the VRN performs the following:

\begin{enumerate}
\item Find the $y$ most probable states.
\item Find the corresponding $y$ \textit{node2vec} representation vectors.
\item Compute the weighted average for these vectors, i.e. belief vector. 
\item Utilize this belief vector as an input to the decoder during the generation.

\end{enumerate}

A simple theoretical analysis of the methods demonstrates an extremely low probability of replication by our approach. 

\begin{proposition}
Given h3 regions $h3_i, h3_{i+1}$, the transition probability $P(h3_{i+1}|h3_i)$ generated by VRN is $ \approx \frac{1}{y}$, where $y$ is the number of top most probable nodes used in the belief vector. 
\end{proposition}
\begin{proof}
Since the belief vector represents the $y$-most probable nodes, it is expected that the next generated node can be one of them. The expected probability of a transition is, therefore, congruous to $\frac{1}{y}$, given $ h3_{i+1} = \sum_{i=1}^{y} Prob(h3_i) \times h3_{emb_i}$. 
\end{proof}

Given training trajectories $T_{train}$ (each of length $m$) that consist of regions $(h3_1,h3_2,\ldots)$, the probability of generating a $S$ with an identical ordering of regions is upper-bounded by $(\max Prob(h3_i, h3_{i+1}))^m$, where $Prob(h3_i, h3_{i+1})$ is the most common transition probability in $T_{train}$. Given $Prob(h3_i, h3_{i+1}) \approx \frac{1}{y}$, the probability of duplicate trajectories is $(\frac{1}{y})^m \to 0$ when $y>1$ and $ m>>1$.

\end{document}